\documentclass[letterpaper]{article} 
\usepackage[preprint]{aaai2027}  
\usepackage[hyphens]{url}  
\usepackage{graphicx} 
\usepackage{natbib}  
\usepackage{caption} 
\usepackage[vlined,ruled,linesnumbered]{algorithm2e}
\usepackage{amssymb}
\usepackage{pifont}
\usepackage{subcaption}

\usepackage{newfloat}
\usepackage{listings}
\DeclareCaptionStyle{ruled}{labelfont=normalfont,labelsep=colon,strut=off} 
\DeclareFloatingEnvironment[fileext=lst,placement={tb},name=Listing]{listing}
\lstdefinelanguage{PDDL}{
  sensitive=false,
  morecomment=[l]{;},
  morekeywords={
    define,domain,problem,requirements,types,constants,predicates,
    functions,action,durative-action,parameters,precondition,effect,
    condition,duration,objects,init,goal,metric,
    and,or,not,exists,imply,
    when,forall
  }
}
\usepackage{booktabs}

\usepackage{amsmath}
\usepackage{amssymb} 
\usepackage{amsthm}
\usepackage{mathtools}
\usepackage{pifont} 
\usepackage{dsfont}
\usepackage{multirow}
\usepackage{mathrsfs} 

\newcommand{\mD}{\mathcal{D}}

\newcommand{\mH}{\mathcal{H}}

\newcommand{\mM}{\mathcal{M}}
\newcommand{\mO}{\mathcal{O}}
\newcommand{\mP}{\mathcal{P}}

\newcommand{\execr}{v}
\newcommand{\hmap}{h_\mathrm{map}} 
\newcommand{\lbound}{\hat{L}}
\newcommand{\mLa}{\mathcal{L}_a}

\newcommand{\score}[2]{#1$\pm$#2}

\DeclareMathOperator{\preop}{pre}
\DeclareMathOperator{\effop}{eff}

\DeclareMathOperator{\softmax}{Softmax}

\DeclareMathOperator{\reset}{reset}
\DeclareMathOperator{\getquantified}{get-new-quantified}

\newtheorem{proposition}{Proposition}
\newtheorem{corollary}{Corollary}

\newcommand{\cmark}{\textcolor{green}{\ding{51}}} 
\newcommand{\xmark}{\textcolor{red}{\ding{55}}}  
\definecolor{darkyellow}{rgb}{0.90, 0.75, 0.0}
\newcommand{\warnmark}{\textcolor{darkyellow}{\ding{115}}} 
\title{Learning Action Models with Conditional and Quantified Effects via Uncertainty-Guided Exploration}

\author {
    Jeffrey Jewett\textsuperscript{\rm 1},
    William Solow\textsuperscript{\rm 1},
    Sandhya Saisubramanian\textsuperscript{\rm 1}
}
\affiliations {

    \textsuperscript{\rm 1}Oregon State University\\
    jewettje@oregonstate.edu,
    soloww@oregonstate.edu,
    sandhya.sai@oregonstate.edu
}

\begin{document}

\maketitle

\begin{abstract}
Accurate action models are critical for effective planning. Existing action-model learning methods largely assume simple action representations or become computationally intractable when learning conditional and quantified effects. We present Online Hypothesis-Driven Conditional Action Model Learning (OHCAM), an online approach for learning action models with conditional and quantified effects from limited interactions with the environment. OHCAM maintains a belief over hypothesized action models and actively selects informative actions to reduce uncertainty by maximizing disagreement among competing hypotheses, while being robust to noisy observations. To enable scalability, OHCAM begins with a small set of simple action model hypotheses and expands to more complex conditions only when the current hypotheses become inconsistent with the data. Experiments on six benchmark planning domains demonstrate that OHCAM is sample efficient in learning action models that solve substantially more tasks than baselines, even with observation noise. We validate OHCAM on two tasks using a Kinova Gen3 robot, demonstrating the real-world applicability of our approach.

\end{abstract}

\section{Introduction}
Autonomous agents require accurate action models for reliable high-level task planning. In complex, real-world environments, actions often have context-dependent outcomes that affect nearby objects. Capturing such behavior requires expressive action models with conditional and quantified effects~\citep{maoPDSketchIntegratedPlanning}. While recent works have demonstrated the feasibility of learning conditional action models~\citep{zhuoLearningComplexAction2010, mordochSafeLearningPDDL2024}, scalability and robustness to noise remains a significant challenge, due to the combinatorial space of possible conditional and quantified effects, hindering real world applicability. 

Consider a robot tasked with cleaning plates (Fig.~\ref{fig:overview}). A robot wiping a stained plate with a sponge expects to clean it, but instead makes it worse if the sponge is already dirty. When a soda can is on the plate, the wiping action knocks it off the plate. Correctly modeling these behaviors require learning both \emph{when} (conditional) the effect occurs and \emph{what additional objects} (quantified) could be affected as a result of the action. Many conditional and quantified effects depend on environment-specific interactions that are unknown before deployment and cannot be fully captured in an offline dataset or anticipated by a human designer. Consequently, the agent must identify and refine these effects \emph{online} through interaction~\citep{lamannaOnlineLearningAction2021}. This inference task becomes more challenging under noisy observations, where incorrect state estimates obscure the true action dynamics, which is common in the deployment of robots~\citep{barfoot2024state}.

Most existing action model learning methods target learning STRIPS~\citep{fikesLearningExecutingGeneralized1972}, where successful actions apply every learned effect regardless of the context~\citep{yangLearningActionModels2007,ainetoLearningActionModels2019,ainetoComprehensiveFrameworkLearning2022}. Recent approaches for learning conditional action models learn offline and reason over the full hypothesis space, but impose restrictions on the number of learnable conditions to address the combinatorial space of effects~\citep{oates1996learning, zhuoLearningComplexAction2010, mordochSafeLearningPDDL2024}. In contrast, an online learner must continually update its action model from a limited number of noisy interactions while remaining computationally tractable. Achieving this requires three key properties: (1) scalability, to learn efficiently in larger domains without pre-defined limits on the maximum complexity; (2) robustness to noisy observations, since real-world state transitions are often imperfectly observed; and (3) sample efficiency, as real-world executions are time consuming.

\begin{figure*}[t!]
    \centering
    \includegraphics[width=0.95\linewidth]{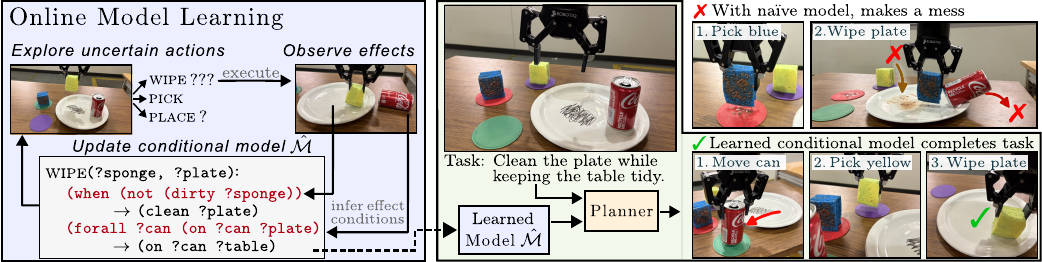}
    \caption{An overview of the online conditional action model learning setting. The agent is given a brief period to interact with the environment before it is evaluated on unseen tasks. By observing when an effect occurs, it learns a model with conditional and quantified effects (left) to then effectively plan to move the can before wiping the plate with the \emph{clean} sponge (right). }
    \label{fig:overview}
    \vspace{-3pt}
\end{figure*}

We present Online Hypothesis-Driven Conditional Action Model Learning (OHCAM) that integrates hypothesis refinement and active exploration for online learning of action models with conditional and quantified effects. OHCAM maintains a posterior (belief) over a compact set of candidate hypotheses using a soft-consistency likelihood that is robust to noisy observations. To infer the maximum a posteriori hypothesis, OHCAM uses a branch-and-bound search that begins with few simple action-model hypotheses and \textit{dynamically expands} to more complex conditions only when supported by the data. This avoids reasoning over the combinatorial hypothesis space, enabling scalability. Further, OHCAM supports sample-efficient exploration by selecting actions that maximize disagreement among plausible hypotheses. 

Evaluation on six benchmark planning domains with conditional effects and quantified variables show that OHCAM achieves substantially higher task completion rates and improved sample efficiency over state-of-the-art action model learning methods. We further validate OHCAM on two tasks with a Kinova Gen3 robot, demonstrating its practical applicability for learning models online in noisy environments.

\section{Related Works}

\paragraph{Action Model Learning (AML)}
The AML problem seeks to infer the preconditions and effects of parameterized actions in symbolic models from observed executions. Most existing AML approaches learn STRIPS action models~\citep{fikesLearningExecutingGeneralized1972}, where the effects are applied unconditionally and are restricted to the objects in the action parameters~\citep{yangLearningActionModels2007, amirLearningPartiallyObservable2008, ainetoLearningActionModels2019, jubaSafeLearningLifted2021b, linToldYouThat2025}. AML algorithms search over logical hypotheses by eliminating candidate preconditions and effects that are inconsistent with observations~\citep{ainetoComprehensiveFrameworkLearning2022}. Noise in real-world observations is a significant limitation of this logical induction as these methods will prune the ground truth formula after a single noisy observation. Several works designed for robustness to noise are restricted to STRIPS effects~\citep{zhuoActionmodelAcquisitionNoisy2013, lamannaActionModelLearning2024}. While effective for many planning domains, these methods cannot capture conditional behaviors or effects on additional objects, limiting their applicability to complex environments.

\paragraph{Models with Conditional Effects and Quantified Variables} Deictic references~\citep{pasulaLearningSymbolicModels2007} implicitly model conditional effects and a limited number of quantified effects with multiple actions.~\citet{chitnisLearningNeuroSymbolicRelational2022} learn such references though clustering, although it learns overly permissive preconditions under noisy observations. LAMP~\citep{zhuoLearningComplexAction2010} explicitly models conditional and quantified formulas but the resulting hypothesis space grows rapidly, limiting LAMP to just one quantified variable. Conditional-SAM~\citep{mordochSafeLearningPDDL2024} learns conditional effects with a pre-specified number of quantified variables, but reasons over exponentially many hypotheses. These offline approaches remain fundamentally limited by the complexity of reasoning over large hypothesis spaces. They are also brittle to noisy observations, where a single inconsistent transition may eliminate the correct hypothesis. Our approach addresses both challenges by maintaining a belief over conditional hypotheses and dynamically expanding only the most promising conditions. 

\begin{table}[t]
    \centering
    \setlength{\tabcolsep}{2pt}
    \renewcommand{\arraystretch}{1.1}
    \resizebox{\columnwidth}{!}{%
    \begin{tabular}{@{}lccccc@{}}
        \toprule
        & Conditional
        & Quantified
        & Noise Robust
        & Online
        & Scalable \\
        \midrule
        Lamanna et al.~(\citeyear{lamannaActionModelLearning2024}) & \xmark & \xmark & \cmark & \xmark & \cmark \\
        \citet{lamannaOnlineLearningAction2021} & \xmark & \xmark & \xmark & \cmark & \xmark \\
        \citet{zhuoLearningComplexAction2010} & \cmark & \cmark & \warnmark & \xmark & \xmark \\
        Pasula et al.~(\citeyear{pasulaLearningSymbolicModels2007}) & \warnmark & \warnmark & \cmark & \xmark & \xmark \\
        \citet{chitnisLearningNeuroSymbolicRelational2022} & \warnmark & \warnmark & \xmark & \xmark & \cmark \\
        \citet{mordochSafeLearningPDDL2024} & \cmark & \cmark & \xmark & \xmark & \xmark \\
        OHCAM (Ours) & \cmark & \cmark & \cmark & \cmark & \cmark \\
        \bottomrule
    \end{tabular}}
    \caption{Comparing action model learning approaches across key capabilities required for use in real-world environments.}
    \label{tab:related-work-features}
\end{table}
\paragraph{Online AML}
Online AML agents select actions to execute and incrementally update the model with the observed outcome~\citep{certickyRealTimeActionModel2014}. As real world data are time consuming to collect, sample efficiency is critical, requiring active exploration of informative actions. Rather than exploring myopically~\citep{rodriguesActiveLearningRelational2012}, a recent approach constructs multi-step plans to reach any state that distinguishes competing STRIPS hypotheses~\citep{lamannaOnlineLearningAction2021}. We adapt this idea to support active learning of conditional effects under noise in observations. Table~\ref{tab:related-work-features} summarizes the key capabilities of the existing action-model learning approaches. 

\section{Problem Formulation}
We consider the problem of learning a symbolic action model with conditional and quantified effects from limited online interactions in an environment. A planning domain $\mM = (\mP, A)$ consists of a set of predicates $\mP$ describing object attributes and relations, and a set of parameterized actions $A$. Action models specify \textit{preconditions} $\mathrm{pre}(a)$ (when the action is applicable) and \textit{effects} $\mathrm{eff}(a)$ (what the action changes in the state). Each action takes a tuple of typed parameters $\mathrm{params}(a) = (?p_1, \dots,  ?p_n)$. A planning task $\tau_\mM = (\mathcal{O},s_0, g) $ consists of a set of objects $\mathcal{O}$, an initial state $s_0$ and a goal condition $g$. An action is grounded by instantiating its parameters with a tuple of unique objects $o \subseteq \mO$, denoted by $a(o)$, where $o_i \neq o_j$ for all $i \neq j$. 

The agent interacts with an environment given the predicates $\mP$ and parameterized action schema, but the preconditions and effects of each action must be learned online. Such scenarios are common in robotics settings where we know what (macro) actions are available but do not know in advance exactly how they affect the environment. 

\begin{figure*}[t]
    \centering
    \includegraphics[width=0.9\linewidth]{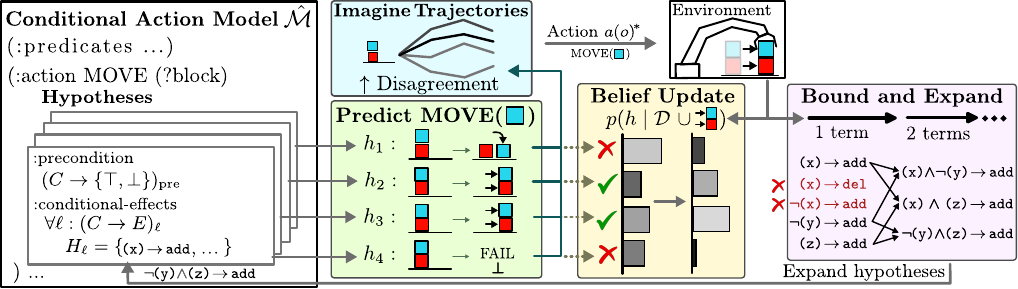}
    \caption{Overview of OHCAM with four hypotheses in a blocks domain: $h_1$ predicts the blue block will move, $h_2$ and $h_3$ predict both blocks will move, and $h_4$ predicts a failure. OHCAM imagines trajectory rollouts under its current belief and executes the action that produces the highest predictive disagreement. Observing that both blocks move, the posterior probability shifts toward $h_2$ and $h_3$, while $h_1$ is reduced, but not eliminated. A more complex hypothesis replaces $h_4$ to maintain diversity.
    }
    \label{figure:ohcam}
\end{figure*}

\paragraph{Conditional and Quantified Effects} To model complex action behaviors, we need an expressive representation of action effects. Unlike STRIPS, we allow actions to have \textit{conditional effects} which depend on the execution context and \textit{quantified effects} which can influence objects beyond the action parameters through \emph{universally quantified variables} (UQVs). A \textit{literal} $\ell(?p_{i_1}, \dots, ?p_{i_m})$ binds action parameters to a predicate $\rho \in \mP$. Let $\mathcal{L}_a$ be the set of all literals bound by $\mathrm{params}(a)$. We define $\mathrm{eff}(a)$ by associating each $\ell \in \mathcal{L}_a$ with a conditional effect $(C \to E)_\ell^a$, where the $C$ is a conjunction of predicate literals and $E$ denotes an add or delete effect on $\ell$. When a grounded action $a(o)$ is executed in state $s$, the effect is applied to the next state $s'$ only if all terms $\ell^c \in C_\ell$ are satisfied: $\ell^c(o) \models s$. 
A quantified effect extends this representation by introducing UQVs that influence objects outside of $\mathrm{params}(a)$ and applied to all objects satisfying the corresponding condition. 

\paragraph{Online Action Model Learning} Given an environment $\Delta$ with states grounded in predicates $\mP$ and a set of actions $a\!\in\!A$ with known parameters $\mathrm{params}(a)$, the objective is to learn an action model $\hat{\mM} = \{\hat{\mathrm{pre}}(a), \hat{\mathrm{eff}}(a)\}_{a\in A}$ where $\hat{\mathrm{eff}}(a)$ may contain conditional and quantified effects. During learning, the agent interacts with the environment by executing actions $a(o)$ and observes $(s, a(o), s', \execr)$ where $s,s'$ denote the states before and after action execution, and $v \in \{ \top, \bot\}$ denotes whether the action succeeded ($\top$) or not ($\bot$). Without an offline dataset, the agent must actively select informative actions that reduce model uncertainty, under a limited budget $B$ and noisy observations. 

\section{Hypothesis-Driven Action Model Learning}
This section outlines our approach, Online Hypothesis-Driven Conditional Action Model Learning (OHCAM), to incrementally learn action models with conditional and quantified effects (Fig.~\ref{figure:ohcam}). OHCAM is made up of two tightly coupled algorithmic components (Alg.~\ref{alg:ohcam}): (i) \emph{dynamic hypothesis expansion} which maintains a belief over plausible action models and expands the hypothesis space only when simpler models become inconsistent with the observed data, thereby avoiding exhaustive reasoning over exponentially many candidates, and (ii) \emph{belief-guided exploration} that enables sample-efficient learning by planning to reach regions that maximally reduce uncertainty over competing hypotheses, while remaining robust to noisy observations. We describe each component in detail below.

\subsection{Dynamic Hypothesis Expansion} 
Let $\mathcal{H}^\mathcal{M}$ be the space of models of the form $\{\hat{\mathrm{pre}}(a), \hat{\mathrm{eff}}(a)\}_{a\in A}$. Learning action models with conditional and quantified effects requires reasoning over an exponentially large space of candidate hypotheses, as each effect may depend on an arbitrary conjunction of predicates. Our objective is to find the hypothesis
in $\mathcal{H}^\mathcal{M}$ that best explains the observed execution data $\mD$. We first discuss this in the offline data setting and then extend it to the online setting.

\paragraph{Estimating Hypothesis Likelihood} With an offline dataset $\mD$, the problem reduces to searching for the maximum a posteriori (MAP) hypothesis $\hmap = \arg\max_h P(h \mid \mathcal{D})$. Directly evaluating this posterior is not scalable as it requires jointly reasoning over all preconditions and conditional effects. Therefore, we factor each model action into a precondition hypothesis set $\mathcal{H}^a_{\mathrm{pre}}$ and a hypothesis set $\mathcal{H}^a_\ell$ for each potential effect literal $\ell\in\mathcal L_a$. The space of hypotheses $\mathcal{H}_\ell^a$ is defined over effect conditions of the form $h = (C_h \to E_h)_\ell^a$, where the \textit{antecedent} $C_h$ is a conjunction of literals from $\mathcal{L}_a \setminus \{ \ell \}$. 
Deterministic action effects apply independently given applicability, yielding an approximate factorization:
\begin{equation*}
P(h^\mM\mid\mathcal D)  
\approx \prod_{a \in A} 
P(h_{\mathrm{pre}}^a\mid\mathcal D_a)  
\prod_{\ell\in\mathcal L_a}  
P(h_\ell^a\mid\mathcal D_a^{v=\top}),
\end{equation*}
conditioned on successful executions, $\mathcal D_a^{v=\top}$. Preconditions learn from all executions $\mathcal D_a$ as they determine applicability. As both $h_\ell^a$ and $h_{\mathrm{pre}}^a$ have the same form, our discussion applies to both and we omit the subscripts: $P(h \mid \mathcal{D})$.

To estimate the probability of a hypothesis given $\mathcal{D}$, we compare the observed effect of each execution $(s,o,s')$ with the hypothesis' prediction. Specifically, each $h$ predicts whether a literal $\ell$ will be added, deleted, or unchanged in the next state, and we compare this prediction with whether $\ell(o)$ \emph{changes} between $s$ and $s'$. Aggregating these comparisons over all execution-literal pairs in $\mathcal{D}$ yields the number of true positives ($\mathrm{TP}_h$), true negatives ($\mathrm{TN}_h$), false positives ($\mathrm{FP}_h$), and false negatives ($\mathrm{FN}_h$). 

To account for observation noise, we use a soft-consistency likelihood so that an incorrect prediction decreases, without eliminating, a hypothesis' posterior probability: 
\[P(\mathcal{D} \mid h) = \epsilon^{\mathrm{FP}_h + \mathrm{FN}_h}, 0 < \epsilon < 1, \] 
where $\epsilon$ is a hyper-parameter denoting how strongly prediction errors are penalized, with larger values tolerating noisier observations.
As multiple hypotheses may explain the data equally, we define a simplicity prior $P(h) = \alpha^{|C_h|}$ penalizing complex antecedent terms. The hypothesis score $L(h)$ is defined as the unnormalized log-posterior $L(h) \propto \log P(h\mid D)$, to balance predictive fit and model complexity:
\begin{equation}
\label{eq:posterior}
L(h) =  |C_h|\log \alpha + (\mathrm{FP}_h + \mathrm{FN}_h) \log \epsilon.
\end{equation}

While $L(h)$ measures a hypothesis' consistency with $\mD$,  identifying the MAP hypothesis $\hmap = \arg \max_{h \in \mathcal{H}} L(h; \mD)$ remains challenging as each hypothesis set scales combinatorially in the number of literals.

\paragraph{Dynamic Expansion}
To address the combinatorial nature of the hypothesis set, OHCAM performs \textit{dynamic hypothesis expansion}, which incrementally refines candidate hypotheses by adding one literal to the antecedent at a time, expanding only those that have the highest potential to improve over the current best hypothesis score (Lines~\ref{line:dynamicexp}-\ref{line:refinehypothesis}, Alg.~\ref{alg:ohcam}). Let $H_\ell$ be a set of expanded hypotheses for potential effect $\ell$. We initialize $H_\ell$ with the simplest hypotheses: $\mathrm{false} \to \texttt{null}$ (no effect), $\emptyset \to \texttt{add}$ (unconditional add), $\emptyset \to \texttt{del}$ (unconditional delete) (Line 1; Alg.~\ref{alg:ohcam}). A conditional hypothesis $h: C_h \to E_h$ is \textit{refined} by adding a literal $\ell'$ (or its negation $\neg \ell'$) to the antecedent, $h': C_h \cup \ell' \to E_h$ (Line~\ref{line:refinehypothesis}). These refinements define a tree in which each node represents a conditional hypothesis and each edge corresponds to adding one antecedent literal. Dynamic hypothesis expansion performs a best-first branch-and-bound search over the tree, bounding $|H_\ell|$ by the complexity of the MAP condition $|C_{\hmap}|$. 

To guide this search and determine which hypotheses should be refined, each hypothesis is associated with an admissible upper bound on the score attainable by any of its descendants. Adding an antecedent literal makes a condition more specific, so the number of false positives cannot increase and the number of false negatives cannot decrease. Optimistically, a sequence of refinements $h \to \dots \to h'$ eliminates all false positives without introducing false negatives. Every refinement  $h' \succ h$ is thus upper-bounded by
\begin{equation}
\label{eq:optimistic}
U(h) = (|C_h|+1)\log \alpha + \mathrm{FN}_h \log \epsilon \geq L(h').
\end{equation}

The search prioritizes refining hypotheses with the maximal $U(h)$. It terminates when no frontier hypothesis has an upper bound exceeding the score $\lbound$ of the current best hypothesis, $\max_{h \in H_\ell} U(h) \leq \lbound$, thereby returning a MAP hypothesis (Line~\ref{line:whileexpand}; Alg.~\ref{alg:ohcam}). While this establishes the correctness, efficiency depends on minimizing the number of hypotheses explored. 

Since the search depth is determined by the complexity of the target condition, the branching factor, which is the number of literals considered for refinement, must be minimized to the extent possible. Rather than refining with every literal in $\mathcal{L}_a$, OHCAM considers only literals that could plausibly improve the current best hypothesis, as not every literal can produce a child hypothesis whose score exceeds the current best hypothesis (e.g. when it predicts too many false negatives). Adding a literal also incurs an additional simplicity penalty, since Eq.~\ref{eq:posterior} favors simpler explanations. A refinement can improve the score only by eliminating enough false positives to offset this penalty. We therefore restrict the \emph{refinement literal set} $R$ 
to only those literals that have an upper bound above the current best score and predict sufficient true negatives. This restricts the branching only to literals that correlate with the effect, which is often much smaller than $|\mLa|$. \textit{Appendix A} includes theoretical analysis of these properties. 

\begin{algorithm}[t]
\small
\caption{OHCAM}
\label{alg:ohcam}
\SetAlgoVlined
\DontPrintSemicolon
\SetNlSty{}{}{}
\newcommand{\comm}[1]{\ \textcolor{gray}{// #1}}
\KwIn{Environment $\Delta$; interaction budget $B$; rollouts $N$; horizon $k$; hypothesis pool size $n$}
Initialize $\mH^\mM \gets \{\mathcal{H}^a_\ell \mid a \in \Delta.A,\ \ell \in \mathcal{L}_a \cup \{\mathrm{pre}\}\}$\; \label{line:initialize}
$s \gets \reset(\Delta)$;\ \ $\mD \gets \emptyset$\; 
\For{$b = 1$ \KwTo $B$}{
    \comm{Belief-guided exploration} \label{line:beliefguide} \\
    \For(\comm{imagine trajectory rollouts}){$m = 1$ \KwTo $N$}{
        $s_1 \gets s$;\ \ $\tilde{\mM} \sim P(h \mid \mD)$ \comm{sample model} \label{line:samplemodel}\\
        \For{$i = 1$ \KwTo $k$}{
            $a(o)_i \sim \softmax(Z(s, a(o)_{1:\eta}))$ \label{line:softgreedy} \\
            $s_{i+1} \gets \tilde{\mM}(s_i,a(o)_i)$\; \label{line:imaginenext}
        }
        $G_m \gets G(s_{1:k}, a(o)_{1:k})$ \comm{Eq.~\ref{eq:score}} \label{line:cumscore}
    }
    $a(o)^\star \gets a(o)_1^{m^\star}$ \comm{first action of best rollout} \label{line:execaction}\\
$\mD \gets \mD \cup \{(s, a(o)^\star, s', \execr)\}$;\ \comm{execute action} \\
    \comm{Dynamic hypothesis expansion} \label{line:dynamicexp} \\
    \For{$(\ell^q, Q, C_Q) \in \getquantified(s,a(o),s')$}{
        $\mathcal{H}^\mM \gets \mH^\mM \cup \mH_{\ell^q}^Q$; initialize with $C_Q$\;
    }
    \ForAll{$\mathcal{H}^a_\ell \in \mH^\mM$}{
        update $L(h; \mD^{\execr=\top})$ for $h \in \mathcal{H}^a_\ell$ \comm{Eq.~\ref{eq:posterior}} \\
        $\hat{L} \gets n\text{-th best } L(h; \mD^{\execr=\top})$ \comm{active bound} \label{line:activebound} \\
        \While{$\exists\, h \in \mathcal{H}^a_\ell : U(h) > \hat{L}$}{ \label{line:whileexpand}
            $\mathcal{H}^a_\ell \gets \mathcal{H}^a_\ell \cup \{(C_h \cup \{\ell'\} \to E_h)\}$ \comm{expand} \label{line:refinehypothesis} 
        }
    }
    $s \gets s'$\;
}
\KwRet{$\hat{\preop_a}, \hat{\effop_a} \gets \arg\max_h P_a(h \mid \mD)$ for $a \in \Delta.A$}
\end{algorithm}

\subsubsection{Quantified Effects} 
So far, we discussed dynamic expansion for conditional effects under a static set of literals $\mathcal{L}_a$ bound only to the parameters of the action. However, actions in complex environments can produce effects involving objects outside of $\mathrm{params}(a)$ 
that must be explained with \textit{universally quantified variables} (UQVs). The addition of UQVs not only expands $\mathcal{L}_a$ with more variables, but multiple candidate UQVs might explain the effect, compounding the complexity of the quantified hypothesis space. To keep UQV learning tractable with our dynamic hypothesis expansion method, we make two assumptions: (i) \emph{identifiability}---all effects matching a quantified literal have a single explanation, and (ii) \textit{locality}~\citep{maoPDSketchIntegratedPlanning}---there exists some relationship, direct or indirect, between the affected object and $\mathrm{params}(a)$. In our cleaning example in Fig.~\ref{fig:overview}, a soda can gets knocked over in the $\textsc{wipe}(\texttt{sponge}, \texttt{plate})$ action because it relates to the plate by $\texttt{on}(\texttt{can}, \texttt{plate})$. 

Formally, consider a grounded action $a(o_1^a, \dots o_n^a)$, where $o^a$ denotes the objects bound to the action parameters. Let $\rho(o_1^e, \dots o_m^e)$ be the action effect involving objects $o_i^e$. If any affected object lies outside the action parameters, $o_i^e \notin o^a$, the effect is represented as a quantified literal 
$\ell^q$ with one or more parameters unbound. Let $\mathbf{Q} = \{Q_1, Q_2, \dots \}$ be the set of candidate UQV sets that could explain the quantified effect. We must bind the unbound parameters of $\ell^q$ to UQVs $Q$ in $\mathbf{Q}$ with corresponding conditions $C_Q$ that only the affected objects satisfy. Under the identifiability assumption, the full space of hypotheses is $\mathcal{H}_{\ell^q}^\mathbf{Q} = \bigcup_{Q \in \mathbf{Q}} \mathcal{H}_{\ell^q}^Q$. For a given UQV set $Q$, the hypothesis space $\mathcal{H}_\ell^Q$ is defined over conjunctions of literals $\mathcal{L}_a^Q$ bound to variables $\mathrm{params}(a) \cup Q$. 

We perform dynamic expansion over sets of hypotheses, favoring simple UQV hypotheses by extending the simplicity prior to penalize more UQVs: $P(h)=\alpha^{|C_h|}\cdot\beta^{|Q|}$, where $\beta \in (0, 1)$ is a penalty parameter on the hypothesis. We rely on the locality assumption to avoid enumerating every possible quantified variable assignment. We formalize the relationship between objects with a hypergraph over the observed state $s$: each object is a vertex, and each grounded predicate $\rho(o_1^e, \dots, o_k^e) \in s$ induces a hyperedge over its arguments (Fig.~\ref{figure:uqv}). Then, $q = o^e \setminus o^a$ is the set of objects that must be explained by a UQV. 

We enumerate all trees in the state graph that connect every object in $q$ (affected objects) to any object in $o^a$ (action parameters), starting with the simplest. A UQV set $Q$ is formed from the nodes of a tree, excluding $o^a$, and the edges become conditions $C_Q$. If an effect cannot be matched to an existing hypothesis set, we initialize the next simplest quantified hypothesis set $\mH_{\ell^q}^Q$ with antecedent $C_Q$. We perform dynamic expansion on each quantified hypothesis set to further refine the antecedent and select the MAP hypothesis from the sets.

\begin{figure}[t]
    \centering
    \includegraphics[height=25.438mm]{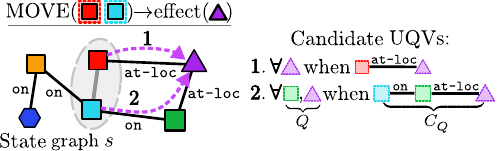}
    \caption{A state graph for UQVs. If an action that intends to move \texttt{red} and \texttt{blue} also affects \texttt{purple}, then we propose sets of UQVs that connect \texttt{purple} to $\{\texttt{red}, \texttt{blue}\}$.}
    \label{figure:uqv}
\end{figure}

\subsection{Belief-Guided Exploration}
So far, we have discussed identifying the MAP hypothesis from an offline dataset $\mathcal{D}$. We now extend this to learning online. 
Since multiple hypotheses can explain the observations seen so far, subsequent action executions should be chosen to maximally distinguish among these competing hypotheses. OHCAM achieves this by selecting actions expected to provide the greatest reduction in posterior uncertainty using Bayesian Active Learning by Disagreement (BALD)~\citep{houlsbyBayesianActiveLearning2011} (Lines~\ref{line:beliefguide}-\ref{line:execaction} in Alg.~\ref{alg:ohcam}). BALD was introduced for classification tasks and seeks to reduce the space of possible models in as few samples as possible. We adapt the objective to our setting. With a finite set of deterministic hypotheses, the BALD utility is the entropy of a weighted vote of predicted outcomes:
\begin{align}
    Z(s, a(o)) =& \mathbb{H}[\sum_h P(\execr, s' \mid s, a(o), h) \cdot P_a(h \mid \mD)].
    \label{eqn:util}
\end{align}
The utility is maximized when plausible hypotheses make different predictions for executing $a(o)$ in state $s$, either about action applicability $\execr$ or about the resulting state $s'$. Assuming independence among different literals, the predictive entropy decomposes into uncertainty over action applicability and uncertainty over individual effects:
\begin{multline}
\mathbb{H}[P(\execr, s' \mid s, a(o))] = \mathbb{H}[P(\execr \mid s, a(o))]\\
 + P(\execr = \top \mid s, a(o))
 \cdot \sum_{\ell} \mathbb{H}[P(s'^{\,\ell(o)} \mid s, a(o))],
\end{multline}
where the effect entropies are discounted by the uncertainty of whether the action is applicable, since the effects are observed only when the action executes successfully. 
BALD requires multiple plausible hypotheses to quantify disagreement. Therefore, we relax the branch-and-bound criterion to maintain the top-$n$ active hypotheses, expanding any hypothesis satisfying $U(h) > L^{(n)}$, where $L^{(n)}$ is the score of the current $n$-th best hypothesis (Line~\ref{line:activebound}; Alg.~\ref{alg:ohcam}).

\vspace{3pt}
\noindent \textbf{Planning Over Imagined Rollouts~} 
Since informative states may require multiple actions to reach, we optimize cumulative information gain over imagined trajectories. Unlike standard rollouts under a fixed dynamics model, imagined trajectories are simulated under posterior-sampled hypotheses to evaluate information gain rather than task reward. Starting from the current state, we simulate $k$-step trajectory by recursively selecting actions and predicting their outcomes using $\tilde{M} \sim P(h \mid \mD)$ (Line~\ref{line:samplemodel}; Alg.~\ref{alg:ohcam}). At each step, we sample $\eta$ actions applicable in $\tilde{M}$, and select one to execute according to a soft greedy bias $a(o)_i \sim \softmax(Z(s, a(o)_{1:\eta}))$ to favor informative trajectories while preserving exploration (Line~\ref{line:softgreedy}; Alg.~\ref{alg:ohcam}). The sampled model predicts the successor state, $s_{i+1}\leftarrow\tilde{M}(s_i,a_i(o_i))$ (Line~\ref{line:imaginenext}; Alg.~\ref{alg:ohcam}). 

Each imagined trajectory is scored by its expected cumulative information gain, with each state's contribution weighted by the probability of reaching it. Let $\Lambda_i$ denote the probability of reaching state $s_i$ along the imagined trajectory, with $\Lambda_1=1$. For $i\geq2$:
\[ \Lambda_i=
\prod_{j=1}^{i-1}
P(\execr_j=\top\mid s_j,a(o)_j)
P(s_{j+1}\mid s_j,a(o)_j).\]
Using Eq.~\ref{eqn:util}, the trajectory score is then (Line~\ref{line:cumscore}; Alg.~\ref{alg:ohcam}): 
\begin{equation}
\label{eq:score}
G(s_{1:k}, a(o)_{1:k})=\sum_{i=1}^{k}
\Lambda_i\cdot Z(s_i,a(o)_i).
\end{equation}

Given a budget $B$, OHCAM repeatedly plans an informative action using receding-horizon rollouts scored by Eq.~\ref{eq:score}, executes the selected action in the environment, and updates the corresponding action's belief by reweighting the posterior (Eq.~\ref{eq:posterior}) and dynamically expanding hypotheses according to Eq.~\ref{eq:optimistic}. This process repeats until $B$ is exhausted, at which point OHCAM returns the MAP action model $\mathcal{M}'$.

\section{Experiment Setup}
\label{sec:experiments}
We evaluate OHCAM on six benchmark planning domains and on two real-world tasks using a Kinova Gen 3 robot.

\paragraph{Baselines} We compare OHCAM with four action model learning methods that support conditional and quantified effects or online learning: (1) LAMP~\citep{zhuoLearningComplexAction2010}; (2) Conditional-SAM~\citep{mordochSafeLearningPDDL2024} with maximum antecedents set to three (maximum depth in our benchmark domains) and ground truth UQV type provided; (3) Cluster\&Intersect~\citep{chitnisLearningNeuroSymbolicRelational2022}; and (4) OLAM~\citep{lamannaOnlineLearningAction2021}. Baselines (1)-(3) are offline methods which learn conditional and quantified effects. OLAM is an online active learning method but only learns STRIPS effects. Hyperparameters and additional details in \textit{Appendix B}.

To enable comparison with the offline methods and further study how the dataset affects the downstream performance, we consider three types of data collection: (1) \emph{OHCAM dataset} collected with our active learning approach; (2) \emph{Random dataset} collected by uniformly selecting actions; and (3) \emph{Expert dataset} that contains demonstration trajectories gathered by a planner with access to the ground truth model, with 20\% random actions for diversity~\citep{sternEvaluatingPlanningModelLearninga}.

\paragraph{Domains} We perform evaluations in simulation using benchmark planning domains with conditional and quantified effects, from~\citet{mordochSafeLearningPDDL2024}: \textit{Satellite}~\citep{long20033rd}, \textit{Briefcase}, \textit{Miconic}~\citep{bacchus2001aips}, and \textit{Maintenance}, \textit{CityCar}, and \textit{Nurikabe}~\citep{vallati20152014}. 
We additionally introduce two new domains for our hardware evaluation. \textit{MagnetBlocks} which is a variant of BlocksWorld with the additional of magnetic blocks that stick together when stacked on top of one another. \textit{SpongeWorld} which is the cleaning task outlined in Figure~\ref{fig:overview} where the robot is tasked with cleaning a plate in the presence of obstructions on the plate and dirty sponges. Additional details for both experiments can be found in \textit{Appendix C}.

\paragraph{Training and Evaluation } 
For each domain, we use available problem file generators~\citep{seipp-et-al-zenodo2022} to generate ten train and evaluation problems, similar to~\citet{sternEvaluatingPlanningModelLearninga} (see \textit{Appendix C}). The evaluation problems are harder, involving more objects or longer plans. We give each method up to \emph{one hour} per domain for learning. 

Our primary evaluation metrics are \textit{task completion rate} and \textit{sample efficiency}. \textit{Task completion rate} is the average fraction of evaluation tasks solved with plans from the learned models. A task is completed when a plan executed open-loop in the environment reaches a goal state without any failed actions. Plans are generated using Fast Downward. For \textit{sample efficiency}, we learn OHCAM and OLAM online and evaluate the completion rate at 10-step intervals. For offline methods, we evaluate on increasing subsets of the data.

\begin{table}[t]
    \centering
    \setlength{\tabcolsep}{1mm}
   \renewcommand{\arraystretch}{1.5}
\resizebox{\columnwidth}{!}{%
    \begin{tabular}{l|c|c|c|c|c}
        \toprule
          Domain  & OHCAM  & OLAM  & Cond-SAM & C\&I & LAMP \\
        \midrule
             Satellite    & \score{96.0}{4.9} & \score{20.0}{40.0} & \textbf{\score{100.0}{0.0}} & \score{98.0}{4.0} & --- \\
            Maintenance  & \textbf{\score{100.0}{0.0}} & \score{0.0}{0.0} & \score{98.0}{4.0} & \score{0.0}{0.0} & --- \\
             Briefcase    & \textbf{\score{100.0}{0.0}} & \score{0.0}{0.0} & \score{60.0}{48.9} & \score{0.0}{0.0} & \score{0.0}{0.0} \\
            Miconic      & \textbf{\score{100.0}{0.0}} & \score{0.0}{0.0} & \score{82.0}{14.7} & \textbf{\score{100.0}{0.0}} & --- \\
            CityCar      & \textbf{\score{70.0}{30.0}} & \score{0.0}{0.0} & \score{10.0}{15.5} & \score{0.0}{0.0} & --- \\
            Nurikabe     & \textbf{\score{96.0}{8.0}} & \score{0.0}{0.0} & \score{0.0}{0.0} & \score{10.0}{0.0} & ---  \\
        \bottomrule
    \end{tabular}}
    \caption{%
        Task completion rate (\%), averaged over 10 problem instances per domain, along with std. deviation. Models were trained over five trials with 100 actions each. OHCAM and OLAM learn online and offline approaches use expert dataset. Bold values are the best for the domain.
    }
    \label{tab:solving}
\end{table}

\begin{figure*}[t]
    \centering
    \begin{tabular}{@{}cccccc@{}}
        \multicolumn{6}{c}{\includegraphics[width=0.95\textwidth]{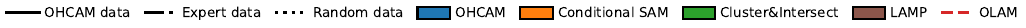}} \\
        \includegraphics[height=1.20in]{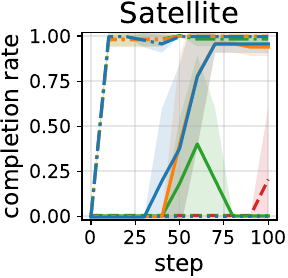} &
        \includegraphics[height=1.20in]{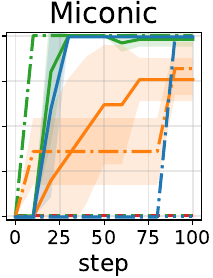} &
        \includegraphics[height=1.20in]{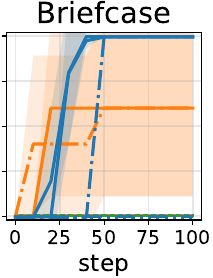} &
        \includegraphics[height=1.20in]{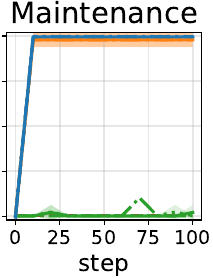} &
        \includegraphics[height=1.20in]{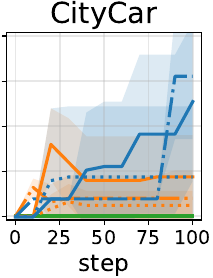} &
        \includegraphics[height=1.20in]{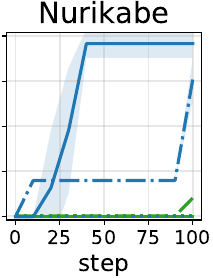}
    \end{tabular}
    \caption{Average problems solved versus number of actions for the six benchmark domains. The color shows the approach and the line type shows the data used. Shading shows standard deviation.}
    \label{fig:sample-efficiency-2607}
\end{figure*}

\begin{figure}[t]
    \centering
    \includegraphics[width=.9\linewidth]{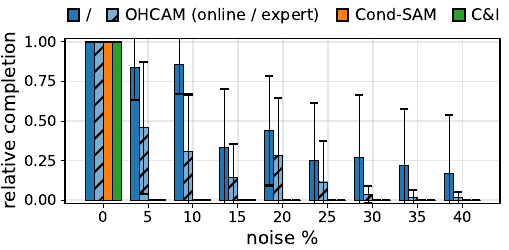}
    \caption{Task completion rate of each approach over 250 samples with increasing noise, relative to noise-less case.}
    \label{fig:noise}
\end{figure}

\section{Simulation Results and Discussion}

\paragraph{Task Completion} Results in Table~\ref{tab:solving} show that OHCAM consistently performs better than or similar to the baselines. OHCAM completed >95\% of tasks on five of six domains with its online active exploration while the next best approach only completed two domains even with expert data. Conditional-SAM learned overly restrictive conditions in larger domains because its strict safety threshold demands more data than the budget. LAMP, which exhaustively enumerates hypotheses, learned extremely slowly, exceeding the training time limit of one hour on many domains. Cluster\&Intersect learned an action for each combination of objects seen in training. However, due to its lack of UQVs, it was unable to generalize to more objects in evaluation tasks. OLAM only succeeded in the \textit{Satellite} domain without conditional effects, highlighting its limits as a STRIPS method to learn expressive models.

\paragraph{Sample Efficiency}
We investigate the sample efficiency of OHCAM compared to the baselines in Figure~\ref{fig:sample-efficiency-2607} using data collected from random, expert, and OHCAM. Results are averaged over 10 problem instances and standard deviation is included over five trials. Across the three data regimens, OHCAM learning from OHCAM-collected data far exceeds the performance of random data, and within 100 actions, OHCAM recovers the performance of models trained on expert demonstrations without the benefit of prior domain knowledge, demonstrating OHCAM's ability to plan to execute informative actions. The quality of OHCAM's actively gathered data is also evident when used by other baselines. Across domains, the baselines trained on OHCAM's online data reach higher problem solving rates than the same baselines trained on random data, indicating OHCAM data collection approach is useful regardless of the learning approach. 

\paragraph{Robustness to Noise}
In real-world settings, there is noise in perception and robot execution. To evaluate robustness to noise, we compare the approaches that can learn conditional effects on the six domains, by randomly corrupting atoms in the observation~\citep{lamannaActionModelLearning2024}. In Figure~\ref{fig:noise}, we vary the noise from zero to 40\% and report the task completion rate relative to learning without noise, averaged over domains. As noise is not limited to objects in the action parameters, this introduces the illusion of quantified effects. Thus, we increase the budget to 250 samples so that the learning approaches can discern between illusory and real effects. Our results show that OHCAM performance degrades gracefully with increasing noise across the six simulation domains. Under 10\% noise, the random fluctuations were insufficient to overcome OHCAM's simplicity prior, preventing false effects. In contrast, the baselines decline sharply with increased noise. Further, OHCAM data collection was more robust to noise than expert data, supporting the claim that OHCAM actively explores to reduce uncertainty.

\section{Robot Hardware Results}
We use a Kinova Gen3 7 DoF arm for manipulation and an Intel RealSense D435i camera to capture RGB and depth images. We use SAM3~\citep{carion2026} to obtain segmentation masks and localize object centers for grasping. Actions are implemented as manipulation skills parameterized by the object locations. More details are in \textit{Appendix D}. We compare OHCAM with the two baselines that showed success learning conditional effects in simulation, i.e. Cluster\&Intersect and Conditional-SAM. Figure~\ref{fig:real_world_results} reports the task completion rate of different approaches, across 10 evaluation tasks in \textit{MagnetBlocks} and five evaluation tasks in \textit{SpongeWorld}. Given the perception and execution noise inherent to robot systems (about 5\% in our system), we collect 250 action samples in both domains using data collected by OHCAM online. Our results show that OHCAM is robust to perception noise, solving all evaluation tasks, compared to the na\"{\i}ve model without conditional effects which is only capable of solving 20\% and 40\% of the task, respectively. The baselines fail as they are unable to address the noise in the data. In MagnetBlocks they learn a model that is incapable of finding a plan for any task due to the higher rate of noise compared to SpongeWorld. A video showcasing robot data collection and evaluation is included in the supplemental materials. 

\begin{figure}[t]
    \centering
    \includegraphics[width=.95\linewidth]{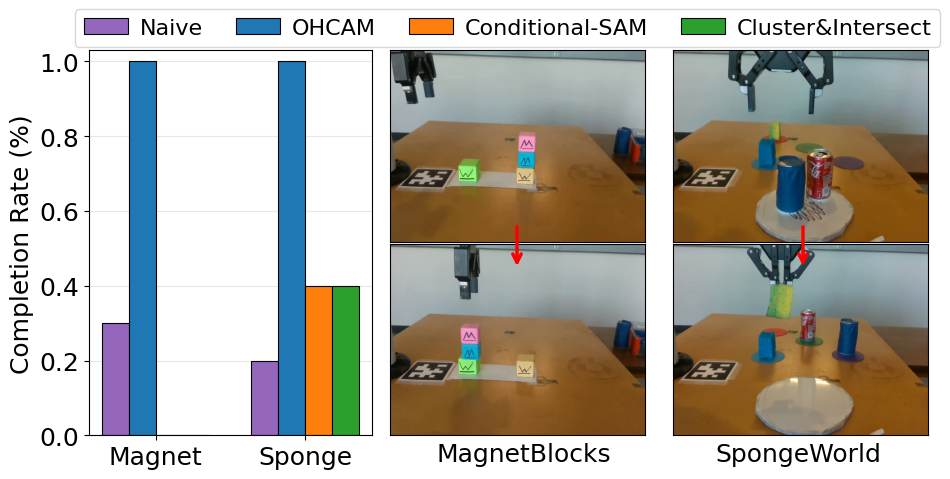}
    \caption{Task completion rate in the two tabletop manipulation tasks, along with an example of start (top) and goal (bottom) configurations.}
    \label{fig:real_world_results}
\end{figure}

\section{Conclusion}
We present OHCAM to learn action models with conditional effects and quantified variables. By combining dynamic hypothesis expansion with belief-guided exploration, OHCAM learns expressive action models from limited interactions while remaining robust to noisy observations. Experiments on benchmark planning domains and a Kinova Gen3 robot demonstrate that OHCAM outperforms state-of-the-art baselines in terms of sample efficiency and downstream planning performance. Future work will extend the framework to partially observable settings and investigate scaling OHCAM to larger relational domains. 

\section*{Acknowledgments} This work was supported in part by NSF award 2543646 and DARPA TIAMAT grant HR0011-24-9-0423.

\clearpage
\bibliography{ref}
\clearpage

\section{Appendix A: Theoretical Results and Proofs}

\renewcommand{\theequation}{A\arabic{equation}}
\setcounter{equation}{0}

In this work, we presented dynamic hypothesis expansion, an algorithm to find the MAP hypothesis of a conditional effect. Here, we establish the correctness of the algorithm.

\begin{proposition}
    For a fixed dataset $\mathcal{D}$ and finite literals $|\mathcal{L}_a|$, best-first dynamic hypothesis expansion with Eq. 2 returns a MAP hypothesis.
\end{proposition}
\begin{proof}
We aim to show in three parts that the hypothesis returned by best-first dynamic hypothesis expansion identifies a MAP hypothesis. First, we show that the MAP hypothesis always remains refinable during branch-and-bound search until it is generated. Second, we show that the MAP hypothesis will eventually be generated. Third, we show that the termination criterion establishes that the returned hypothesis is the MAP.

To begin with, let $h_\mathrm{map} = \arg\max_{h\in\mathcal{H}} L(h; \mathcal{D}) = (C_\mathrm{map} \to E_\mathrm{map})$  be a MAP hypothesis. Let $G=L(h_\mathrm{map})$ be the corresponding hypothesis score. Let each frontier hypothesis have an admissible bound 

\begin{equation*}
U(h) \geq \max_{h' \succ h} L(h')
\end{equation*}
for any descendant $h' \succ h$. Let $H_\ell$ be the set of generated hypotheses. The branch-and-bound algorithm maintains a frontier of generated hypotheses that have not yet been selected for refinement, a current best score $\hat{L} = \max_{h \in H_\ell} L(h)$, and only refines nodes that satisfy $U(h) > \hat{L}$. 

Let $h^0 = (\emptyset \to E_\mathrm{map})$ be the unconditional hypothesis with consequent $E_\mathrm{map}$. All consequents are initialized at the start, so $h^0 \in H$. By a series of refinements with each $\ell' \in C_\mathrm{map}$, $h^0 \to h^1 \to \dots \to h^{|C_\mathrm{map}|-1} \to h_\mathrm{map}$, each $h^i$ is an ancestor of $h_\mathrm{map}$, which means that $U(h^i) \geq L(h_\mathrm{map})=G$. Until $h_\mathrm{map}$ has been generated, there will always remain one of $h^i$ in the frontier, since before $h_\mathrm{map}$ is generated, $\hat{L} < G \leq U(h^i)$. 

Next, under the assumption of a finite set of literals $|\mathcal{L}_a|$, the hypothesis space is finite. Because $h_\mathrm{map}$ is reachable by $|C_\mathrm{map}|$ single-literal refinements and there is always a frontier ancestor on this path, $h_\mathrm{map}$ is eventually generated. Then since the MAP hypothesis by definition has the best score, then $\hat{L} \gets G$, so the new current best is a MAP hypothesis.

Finally, we use the termination criteria to show that $h_\mathrm{map}$ will remain the current best. The search terminates when all frontier hypotheses have the potential for refinement no greater than the current best hypothesis, $U(h) \leq \hat{L}$ . By admissibility, every ungenerated descendant of every frontier node, $h' \succ h$ satisfies $L(h') \leq U(h) \leq \hat{L}$, and the current best is only replaced when some $L(h) > \hat{L}$. Therefore, no ungenerated hypothesis can become the new best, so when the algorithm terminates, a MAP hypothesis is returned.
\end{proof}

Next, we establish that the hypotheses generated are bounded in depth by the MAP hypothesis, meaning the search will automatically be bounded by complexity suggested by the data rather than according to some pre-defined maximum depth hyperparameter. Further, we show that we can safely reduce the branching factor to a refinement literal set thresholded by the current best hypothesis.
\begin{proposition}
For a fixed dataset $\mathcal{D}$, let $\hmap$ be the hypothesis returned by dynamic hypothesis expansion, $\zeta = \log\alpha / \log\epsilon$ be the complexity-to-noise log ratio, and $R_W=\{\ell^r\in\mathcal{L}_a \mid U(h_{\ell^r})>\lbound_W \land \mathrm{TN}_{h_{\ell^r}} > \zeta \}$ be the refinement literal set after a warm-up period of $W$ hypotheses are expanded with best score $\lbound_W$. Then the maximum depth of an expanded hypothesis is $d_{\mathrm{map}} = \left\lfloor |C_{\hmap}| + (\mathrm{FP}_{\hmap}+\mathrm{FN}_{\hmap}) \zeta^{-1} \right\rfloor$. Consequently, the total number of hypotheses expanded is

\[ |H_\ell| \leq W + \sum_{d=0}^{\min\{d_{\mathrm{map}},|R_W|\}} \binom{|R_W|}{d}. \]
\end{proposition}
\begin{proof}
We bound the number of expanded hypotheses in three steps. First, we show that every expanded hypothesis has depth at most $d_\mathrm{map}$. Second, we show that restricting refinement to the literal set $R_W$  preserves the MAP hypothesis. Finally, we count the number of hypotheses that can be generated under these two restrictions to obtain the bound.

\noindent \textbf{Step 1:} Let $G=L(\hmap)$ denote the MAP score. The search expands hypotheses in decreasing order of $U(h)$ and maintains the current best hypothesis score $\lbound$. Since $\lbound$ increases monotonically and only hypotheses satisfying $U(h) > \lbound$ are expanded, every hypothesis selected for refinement satisfies $U(h)\geq G$. Consider a hypothesis $h$ whose refinement produces a child of depth $d=|C_{h}|+1$. Let $\zeta=\frac{\log\alpha}{\log\epsilon}$ denote the complexity-to-noise log ratio, i.e. minimum number of prediction errors that must be eliminated to justify adding one literal. Since $\mathrm{FN}_{h}\log\epsilon\leq 0$, using Eqn. 2:
\begin{align*}
G\leq U(h) =d\log\alpha+\mathrm{FN}_{h}\log\epsilon \leq& \,d\log\alpha. \\
\end{align*}
Since $\log\alpha <0$ and substituting $G$ and $\zeta$, every generated child has depth at most $d_{\mathrm{map}} =\left\lfloor |C_{\hmap}| +(\mathrm{FP}_{\hmap}+\mathrm{FN}_{\hmap}) \zeta^{-1} \right\rfloor$. 

\noindent \textbf{Step 2:} Let $R_W=\{\ell^r\in\mathcal{L}_a \mid U(h^r)>\lbound_W \land \mathrm{TN}_{h_{\ell^r}} > \zeta \}$ be the refinement literal set after $W$ hypotheses have been expanded, with single-literal hypothesis $h^r = (\ell^r \to E_\mathrm{map})$. For every $\ell^r \in C_\mathrm{map}$, $h^r$ is an ancestor of $\hmap$, so $U(h^r) \geq L(\hmap)$. Now consider the hypothesis with $h^- = (C_\mathrm{map} \setminus \{ \ell^r \} \to E_\mathrm{map})$. Since $|C_\mathrm{map}| - |C_{h^-}|=1$, the inequality $L(\hmap)-L(h^-)>0$ implies that $(\mathrm{FP}_{h^-}-\mathrm{FP}_{\hmap})-(\mathrm{FN}_{\hmap}-\mathrm{FN}_{h^-})>\zeta$. Therefore, $\mathrm{FP}_{h^-}-\mathrm{FP}_{\hmap}>\zeta$.
Thus removing $\ell^r$ decreases the number of true negatives by at least $\zeta$. Therefore, $C_{\hmap} \subseteq R_W$. 

\noindent \textbf{Step 3:} Every remaining hypothesis corresponds to a conjunction of literals from $R_W$ with 
depth at most $d_{\mathrm{map}}$. The number of such conjunctions is at most $\sum_{d=0}^{\min\{D_{\mathrm{map}},|R_W|\}} \binom{|R_W|}{d}$. Any  hypothesis in $W$ that contains a literal $\ell'$ not in $R_W$ will not be refined, so there are at most $W$ extra hypotheses expanded containing literals not in $R_W$, yielding at most \[ |H_\ell| \leq W+ \sum_{d=0}^{d_{\mathrm{map}}} \binom{|R_W|}{d}. \]
\end{proof}

Finally, we bound the complexity of the MAP hypothesis by the true conditional effect in the environment. This means that it will not needlessly search for complex explanations of simple conditional or unconditional effects.

\begin{corollary}
In the noiseless case, the hypothesis returned by dynamic hypothesis expansion $h_\mathrm{map}$ is bounded in complexity by the complexity of the true environment conditional effect $h^*$, $|C_\mathrm{map}| \leq |C^{*}|$. Consequently, dynamic hypothesis expansion does not generate any hypothesis deeper than $d^*=|C^*|$. With refinement literal set $R_W$, then 

\[ |H_\ell| \leq W + \sum_{d=0}^{d^*} \binom{|R_W|}{d}. \]
\end{corollary}
\begin{proof}
Assume there is no noise in observation or execution. Let $h^*: (C^* \to E^*)_\ell^a$ be the simplest hypothesis (minimal $|C^*|$) that perfectly predicts every action execution in the environment, $(s^\ell, a(o), s'^\ell) \gets \Delta$ from every state reachable by any initial task state $s_0^\tau$ . Let $h_\mathrm{map}$ be the hypothesis returned by dynamic hypothesis expansion, which is MAP by Proposition 1. By the definition of the MAP hypothesis, $L(h_\mathrm{map}) \geq L(h^*)$. Let $\mathcal{E}_h=\mathrm{FP}_h+\mathrm{FN}_h$. Expanding it out and simplifying:

\begin{equation*}
|C_\mathrm{map}| \log \alpha + \mathcal{E}_\mathrm{map}\log\epsilon \geq |C^*|\log\alpha + \mathcal{E}^* \log\epsilon
\end{equation*}
since $h^*$ is defined as a perfect prediction, then $\mathcal{E}^*=0$. Dividing by $\log\alpha$, 

\begin{equation}
    \label{eq:cstarbound}
|C_\mathrm{map}| + \mathcal{E}_\mathrm{map}\log\epsilon / \log\alpha \leq |C^*|
\end{equation}
Since $\log\epsilon < 0$,  the hypothesis returned by dynamic hypothesis expansion is bounded in complexity by the true environment conditional effect $|C_\mathrm{map}| \leq |C^*|$. 

By Proposition 2, the depth of hypotheses generated in the search is bounded by $|C_\mathrm{map}|$. We now establish that the maximum depth of a generated hypothesis is no more than $|C^*|$.

From Proposition 2, $d_\mathrm{map} = \lfloor |C_\mathrm{map} + \mathcal{E}_\mathrm{map}\log\epsilon / \log\alpha \rfloor$. Straightforwardly from Eq.~\ref{eq:cstarbound}, then $d^* = |C^*| \geq d_\mathrm{map}$. Therefore, dynamic hypothesis expansion will generate no hypothesis deeper than $d^*$ and the number of hypotheses generated with respect to the refinement literal set $R_W$ after $W$ hypotheses have been expanded, is 

\[ |H_\ell| \leq W + \sum_{d=0}^{d^*} \binom{|R_W|}{d}. \]
\end{proof}

\section{Appendix B: Hyperparameters and Additional Results}
\subsection{Appendix B.1: Algorithm Hyperparameters}

\noindent \emph{1. Likelihood and prior:}  The parameter $\epsilon$ in the soft-consistency likelihood has a natural interpretation in terms of prediction noise. Let $\xi$ be the probability that any hypothesis will make a false prediction due to noise. Then a natural likelihood is

\begin{equation*}
P(\mD \mid h) \propto (1-\xi)^{\mathrm{TP}_h+\mathrm{TN}_h}\xi^{\mathrm{FP}_h+\mathrm{FN}_h}.
\end{equation*}

Each hypothesis in a set has the same amount of data, $|\mD|=TP+TN+FP+FN$, so we can substitute and divide off a common term, $(1-\xi)^{|\mD|}$. Then with $\epsilon = \frac{\xi}{1-\xi}$,
$$
P(D | h) \propto \left(\frac{\xi}{1-\xi}\right)^{FP_h+FN_h} = \epsilon^{FP_h+FN_h}
$$
So for an expected amount of noise $\xi$, choose $\epsilon=\frac{\xi}{1-\xi}$. For hardware experiments, we use $\epsilon=0.1$, which corresponds to 11\% noise, a reasonable upper limit. For our noiseless experiments, setting $\epsilon$ to any very small value $\epsilon \ll \alpha$ effectively prunes hypotheses after a single missed prediction.

The simplicity prior $\alpha$ captures how much less probable a conditional hypothesis is compared to a hypothesis with one fewer antecedent. Across planning domains, conditional effects tend to be uncommon, and effects with two antecedent terms are even much less so. It is difficult to empirically estimate the frequency of conditional effects in a domain of interest in advance, but we set $\alpha=0.1$ to represent that conditional effects are sparse but not too unlikely.

\noindent \emph{2. Online Planning Budget:} In active learning, there is a tradeoff between spending more time searching for an informative action and executing a less informative action quicker.
We therefore use a maximum time budget $T_\mathrm{max}$ to imagine rollouts and plan, after which the agent must make a decision. In manipulation domains, a skill execution typically takes five to 20 seconds, so we set $T_\mathrm{max}=5\,\mathrm{s}$. We additionally cap planning at $N=1000$ imagined trajectories, and OHCAM returns the best action available when either budget is exhausted. 

\noindent \emph{3. Rollouts:} The rollout parameters balance planning depth against exploration. Longer rollouts reason about richer interactions but reduce the total number of trajectories that can be evaluated within the planning budget.
We set the rollout length to $k=4$ which is enough for a manipulation agent to pick and place two objects. This allows the agent to plan and test certain interactions between objects such as placing the can on the plate and wiping. At each rollout state, $\eta=6$ candidate actions are sampled, providing several opportunities for informative actions to be selected while maintaining fast rollout generation. Additionally we add a small amount of Gaussian noise to the cumulative score with scale $\gamma=10^{-4}$ to avoid biased action selection when no rollout finds significant disagreement.

\noindent \emph{4. Hypothesis diversity:}
OHCAM expands the top $n$ hypotheses to have a diverse set of hypotheses for disagreement. Having a higher $n$ means that more of the hypothesis space is represented in the active set, giving a better estimation of the probability and more disagreement signal at the cost of expanding. But since the BALD utility is a function of $n$, it can impact performance. We use $n=128$ which allows the agent to roll out many trajectories within the limit.

\subsection{Appendix B.2: Additional Results}
This section presents results with two metrics, in addition to the evaluation metrics and results reported in the main text: (1) \emph{predictive power} metrics~\citep{sternEvaluatingPlanningModelLearninga}, and (2) \emph{number of hypotheses expanded}. 

\paragraph{Predictive power}
Predictive power estimates how well each learned model predicts action outcomes. This metric is important because a model might make false predictions and yet succeed. For example, Cluster\&Intersect on the \textit{Miconic} domain underestimated the amount of passengers that would board, yet still succeeded because there was no penalty for visiting the same floor multiple times. Predictive power is measured against a set of evaluation transitions, and factored into predicted applicability and predicted effects, corresponding to preconditions and effects. In Table~\ref{tab:predictive_power}, we report the predictive power for both action applicability and effects. 

\begin{table}[t]
    \centering

    \label{tab:predictive-results}
    \scriptsize
    \setlength{\tabcolsep}{2pt}
    \begin{tabular}{ll ccccc}
        \toprule
        \multicolumn{2}{l}{\textbf{Approach}}
            & OHCAM & OLAM & Cond-SAM & C\&I & LAMP \\
        \midrule
        \multirow{6}{*}{\rotatebox[origin=c]{90}{App. Precision (\%)}}
            & Satellite    & \textbf{\score{100.0}{0.0}} & \textbf{\score{100.0}{0.0}} & \textbf{\score{100.0}{0.0}} & \textbf{\score{100.0}{0.0}} & --- \\
            & Maintenance  & \textbf{\score{100.0}{0.0}} & \textbf{\score{100}{0.0}} & \textbf{\score{100.0}{0.0}} & * & --- \\
            & Briefcase    & \textbf{\score{100.0}{0.0}} & \score{95.9}{0.8} & \textbf{\score{100.0}{0.0}} & * & \textbf{\score{100.0}{0.0}}\\
            & CityCar      & \textbf{\score{100.0}{0.0} }& \textbf{\score{100.0}{0.0}}  & \score{99.9}{0.1} & * & --- \\
            & Miconic      & \textbf{\score{100.0}{0.0} }& \textbf{\score{100.0}{0.0} }& \textbf{\score{100.0}{0.0} }& * & --- \\
            & Nurikabe     & \textbf{\score{100.0}{0.0}} & \textbf{\score{100.0}{0.0}} & \textbf{\score{100.0}{0.0}}  & * & --- \\
        \midrule
        \multirow{6}{*}{\rotatebox[origin=c]{90}{App. Recall (\%)}}
            & Satellite    & \textbf{\score{100.0}{0.0}} & \score{60.0}{0.0} & \score{80.3}{2.1} & \score{80.0}{0.0} & --- \\
            & Maintenance  & \textbf{\score{100.0}{0.0}} & \textbf{\score{100.0}{0.0}} & \textbf{\score{100.0}{0.0}} & *  & ---\\
            & Briefcase    & \textbf{\score{100.0}{0.0}} & \score{97.3}{0.5} & \textbf{\score{100.0}{0.0}}  & * & \score{0.0}{0.0} \\
            & CityCar      & \score{97.6}{2.8}  & \score{0.0}{0.0} & \textbf{\score{98.7}{2.0}}  & * & --- \\
            & Miconic      & \textbf{\score{100.0}{0.0}} & \textbf{\score{100.0}{0.0}} & \textbf{\score{100.0}{0.0}} & * & --- \\
            & Nurikabe     & \textbf{\score{100.0}{0.0}} & \textbf{\score{0.0}{0.0}} & \textbf{\score{100.0}{0.0}}  & * & ---\\
        \midrule
        \multirow{6}{*}{\rotatebox[origin=c]{90}{Effect Precision (\%)}}
            & Satellite    & \textbf{\score{100.0}{0.0}} & \textbf{\score{100.0}{0.0}} & \textbf{\score{100.0}{0.0}} & \textbf{\score{100.0}{0.0}} & --- \\
            & Maintenance  & \textbf{\score{100.0}{0.0}} & \textbf{\score{100.0}}{0.0} & \textbf{\score{100.0}{0.0}} & * & --- \\
            & Briefcase    & \textbf{\score{100.0}{0.0}} & \score{88.6}{10.3} & \textbf{\score{100.0}{0.0}} & * & \textbf{\score{100.0}{0.0}}\\
            & CityCar      & \textbf{\score{100.0}{0.0}} & \textbf{\score{100.0}{0.0}}  & \textbf{\score{100.0}{0.0}} & * & --- \\
            & Miconic      & \textbf{\score{100.0}{0.0}} &  \textbf{\score{100.0}{0.0}} & \textbf{\score{100.0}{0.0}} & * & --- \\
            & Nurikabe     & \textbf{\score{100.0}{0.0}} & \textbf{\score{100.0}{0.0}} & \textbf{\score{100.0}{0.0}}  & * & --- \\
        \midrule
        \multirow{6}{*}{\rotatebox[origin=c]{90}{Effect Recall (\%)}}
            & Satellite    & \score{99.6}{0.0} & \score{60.0}{0.0} & \textbf{\score{100.0}{0.0}} & \score{79.6}{0.0} & --- \\
            & Maintenance  & \textbf{\score{100.0}{0.0}} & \score{43.0}{0.0} & \textbf{\score{100.0}{0.0}} & *  & ---\\
            & Briefcase    & \textbf{\score{100.0}{0.0}} & \score{69.3}{10.7} & \textbf{\score{100.0}{0.0}} & * & \score{0.0}{0.0} \\
            & CityCar      & \textbf{\score{99.5}{1.0}}  & \score{0.0}{0.0} & \score{0.0}{0.0}  & * & --- \\
            & Miconic      & \textbf{\score{100.0}{0.0}} & \score{66.6}{0.0} & \textbf{\score{100.0}{0.0}} & * & --- \\
            & Nurikabe     & \textbf{\score{100.0}{0.0}} & \score{0.0}{0.0} & \textbf{\score{100.0}{0.0}}  & * & ---\\
        \bottomrule
    \end{tabular}
        \caption{%
        Applicability and effect predictive power metrics scored by precision (\%) and recall (\%) across six benchmark domains. OHCAM and OLAM use online data, while the offline methods use the expert dataset. Results are reported as mean$\pm$std. deviation over five trials. A dash (--) indicates the method timed out on that domain. An asterisk (*) indicates the score could not be computed due to non-determinism.
    }
    \label{tab:predictive_power}
\end{table}

\vspace{3pt}
\noindent \emph{Applicability prediction}: Using the true positives ($\mathrm{TP}_\top$), true negatives ($\mathrm{TN}_\top$), false positives ($\mathrm{FP}_\top$), and false negatives ($\mathrm{FN}_\top$), the precision and recall for preconditions are defined as follows. 

\begin{equation*}
    P_\top(a) = \frac{\mathrm{TP}_\top}{\mathrm{TP}_\top + \mathrm{FP}_\top} ; \quad R_\top(a) = \frac{\mathrm{TP}_\top}{\mathrm{TP}_\top + \mathrm{FN}_\top},
\end{equation*}

with $P_\top(a) =1, R_\top(a)=0$ when $\mathrm{TP}_\top = \mathrm{FP}_\top=0$. 

\vspace{3pt}
\noindent \emph{Effect prediction}: 
True positives for action effects are defined over the number of ground atoms that were correctly predicted to change. True negatives, false positives, and false negatives are defined analogously. 

\begin{align*}
    \mathrm{TP}_{\mathrm{eff}}(s, a) &= |(\hat{s}' \setminus s) \cap (s' \setminus s)| \\
    \mathrm{TN}_{\mathrm{eff}}(s,a) &= |(s \cap \hat{s}' \cap s')| \\
    \mathrm{FP}_{\mathrm{eff}}(s,a) &= |(\hat{s}' \setminus s) \setminus s'| \\
    \mathrm{FN}_{\mathrm{eff}}(s,a) &= |(\hat{s}' \cap s) \setminus s'|.
\end{align*}
The effect precision and recall are defined as with applicability. 

Finally, we average the applicable and effect precisions and recalls over all states and actions.

We were unable to reliably measure the predictive power of Cluster\&Intersect with this metric because it learns multiple action branches per action without mutually exclusive preconditions. This means that at planning time the agent can choose between outcomes, inducing non-determinism, whereas the true environment has a single deterministic outcome. On \textit{Satellite}, there were no conditional effects, so it learned a deterministic model. Additionally, despite Conditional-SAM predicting 100\% of the evaluation dataset on multiple domains, it often timed out due to complex, quantified learned preconditions.

\paragraph{Number of hypothesis expanded} As we scale to harder domains, the exploding hypothesis space limits how many conditions and quantified variables can be learned, limiting the expressiveness. A good approach should enumerate as few hypotheses as possible. In Corollary 1, we establish that OHCAM expands hypotheses with respect to the depth of the true environment conditional effect, rather than instantiating all hypotheses within a fixed antecedent depth. To validate this empirically, Fig.~\ref{fig:hypotheses} reports the total number of hypotheses, across all effects, generated by OHCAM versus the latest conditional action model-learning method, Conditional-SAM. To simulate not knowing the complexity of an environment in advance, we report results for Conditional-SAM with a maximum of $d=1, 2, 3$ antecedent terms. OHCAM is trained offline with $n=1$ to find only the MAP hypothesis.

Our results show that OHCAM generates fewer hypotheses in every domain, despite not originally being given the ground truth quantified variables. This is most notable on the \textit{Nurikabe} domain, where OHCAM learned a depth-3 conditional effect and still generated fewer hypotheses than depth-1 Conditional-SAM. Conditional-SAM exploded to 637 thousand hypotheses with $d=3$, showing that full enumeration of the hypothesis space quickly becomes intractable. These results reinforce the theoretical results in Corollary 1.

\subsection{Appendix B.3: Description of Baselines}

For our baselines, we selected action model learning methods that support conditional effects, quantified effects, or online active learning. 

\begin{enumerate}
    \item LAMP~\citep{zhuoLearningComplexAction2010} learns complex action models with expressive model features using Markov Logic Networks. For fair comparison, we restrict it to only learn conditional and quantified effects. We use the recommended threshold of $0.5$. We were unable to find a publicly available implementation, so we re-implemented it as described in the paper. The prohibitively slow performance observed in our experiments may be attributed in part due to a lack of optimization, but even the learning in the original paper takes substantially longer than any other baseline.
    \item Conditional-SAM~\citep{mordochSafeLearningPDDL2024} extends safe action model learning to conditional and quantified effects.
    \item Cluster\&Intersect is the symbolic learning component of \citet{chitnisLearningNeuroSymbolicRelational2022} that partitions the dataset by observed effect, creating multiple action branches. During execution, we map these action branches back to the environment action.
    \item OLAM~\citep{lamannaOnlineLearningAction2021} is a STRIPS-only online learning method that constructs plans to reach informative states.

Except for LAMP, we use official implementations by the authors.
\end{enumerate}

\section{Appendix C: Data Generation and PDDL Domains}
\subsection{Appendix C.1: Data Generation}

Data for learning action models are generated by adapting the data generation procedure from \citet{sternEvaluatingPlanningModelLearninga} to the online setting, which we outline here.

For each domain, we create a dataset of 10 training and 10 evaluation problem instances. The evaluation instances selected are harder to solve than the training instances. We use existing generators~\citep{seipp-et-al-zenodo2022} to generate problem instances until we get 10 problems that can be solved within a 60 second budget by Fast Downward on a machine with 32GB RAM with the greedy heuristic used by \citet{sternEvaluatingPlanningModelLearninga}. We reference problem files in \citet{ai_planning_classical_domains} for guidance on selecting problem-generation parameters. 

Online learning agents interact with each training problem for 25 steps. Therefore, if the training budget is 100 steps, then the agent interacts with four problem instances. Experiments with 250 steps use all problem instances. Similarly, a random dataset is created by selecting random actions with the same budget as online agents.

Additionally, we gather an expert dataset by following plans generated with the ground truth model over the training problems. In order to gather diverse data, we interleave random applicable actions with 20\% probability, replan, and continue execution.

\begin{figure}
    \centering
    \includegraphics[width=0.97\linewidth]{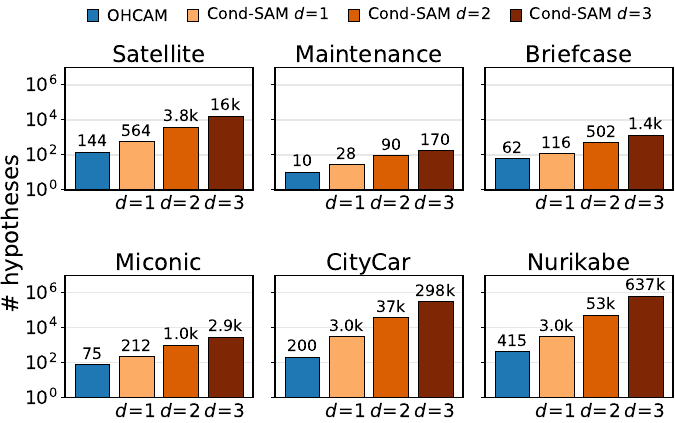}
    \caption{Total number of hypotheses generated by OHCAM's dynamic expansion compared to Conditional-SAM at fixed depth values $d=1,2,3$. y-axis is log-scale.}
    \label{fig:hypotheses}
\end{figure}

\subsection{Appendix C.2: Simulation Domains}
In order to make state-of-the-art comparisons, we use the suite of classic planning domains compiled by \citet{mordochSafeLearningPDDL2024}, which is the most recent work in learning action models with conditional effects and quantified variables. The domains include \textit{Satellite}~\citep{long20033rd}, \textit{Briefcase} and \textit{Miconic}~\citep{bacchus2001aips}, and \textit{Maintenance}, \textit{CityCar}, and \textit{Nurikabe}~\citep{vallati20152014}. These domains range in number of actions, predicates, and conditional and quantified effects. We summarize the features of the domains in Table~\ref{tab:domains}. 

\begin{table*}[!ht]
    \centering
    \begin{tabular}{|l|l|l|l|l|l|l|l|l|l|}
    \hline
        ~ & $|A|$ & $\max |\mathcal{L}_a|$ & $\max |\mathcal{L}_a^Q|$ & $\max |C_\ell|$ & Total CE & Total UQV & Max UQV \\ \hline
        Satellite & 5 & 8 & 8 & 0* & 0* & 0 & 0 \\ \hline
        Maintenance & 1 & 1 & 3 & 1 & 1 & 1 & 1 \\ \hline
        Briefcase & 3 & 3 & 5 & 1 & 1 & 1 & 1 \\ \hline
        Miconic & 3 & 4 & 5 & 2 & 3 & 3 & 1 \\ \hline
        CityCar & 7 & 14 & 14 & 1 & 2 & 2 & 1 \\ \hline
        Nurikabe & 4 & 21 & 31 & 3 & 3 & 3 & 1 \\ \hline
    \end{tabular}
    \caption{Benchmark domain conditional effect and quantified variable features: \# actions $|A|$, maximum \# action-bound literals $\max |L_a|$, \# literals with optimal quantified variables $\max |L_a^Q|$, maximum \# antecedent terms $\max |C_\ell|$, total conditional effects CE, total UQVs, and maximum \# UQVs in one effect, and number of disjunctive effects. \*\textit{Satellite} conditions are redundant.}
    \label{tab:domains}
\end{table*}
	
\textit{Satellite} models a collection of satellites that must rotate towards targets and allocate power to on-board instruments to take images. All conditions are of the form \texttt{(when (not (x)) (x))}, which is behaviorally identical to an unconditional effect. Thus, this domain supports comparisons to STRIPS-only learning methods (e.g. OLAM).

\textit{Briefcase} models the transportation of objects by putting them in a briefcase. Moving involves universally quantifying over all objects in the briefcase to place them in the location specified in the action parameters. This domain has a maximum of three action parameter-bound literals, making it simple for learning conditions.

\textit{Miconic} defines actions for an elevator that must move between floors and allow passengers to board. The \textsc{stop} action involves multiple quantified effects, involving conditions such as only deboarding passengers when they reach their destination.

\textit{Maintenance} is the smallest domain, featuring only a single action with only three possible condition literals. The domain involves scheduling when and where mechanics should work. A quantified variable selects all airplanes at that location on that day to be repaired.

\textit{CityCar} is a domain where the agent works as a city planner managing road infrastructure and traffic. The domain is more complex, with seven actions, 14 possible conjunctive terms, and two quantified effects. This domain turned out to be the most challenging for learners because of a quantified effect that only happens when a road is destroyed while a car is on it, which is a rare interaction to encounter. Additionally, in our testing we observed that overly complex preconditions affect task completion rate: there were some tasks that OHCAM, which learns minimal precondition sets, was able to solve, where the same model with a safe overapproximation timed out.

\textit{Nurikabe} is a simplified planning version of a tile puzzle of the same name. This domain has the largest predicate space, with a maximum of 31 possible conjunctive terms. The maximum condition involves three literals, which \citet{mordochSafeLearningPDDL2024} highlights as a significant scaling challenge.

All domains are available in \citet{ai_planning_classical_domains}, except the conditional version of \textit{Satellite} is found at \citet{potassco_satellite_instance}. As in \citet{mordochSafeLearningPDDL2024}, we ignore action costs. 

\subsection{Appendix C.3: Robot Hardware Domains}
We evaluate OHCAM and the baselines on two domains on real robot hardware. We detail the specifics of these domains and their conditional effects below. The na{\"\i}ve action models we use and the learned action models can all be found in the models folder of the supplemental materials. 

\paragraph{MagnetBlocks}
The magnet blocks domain is similar to the \textit{BlocksWorld} classical planning domain in which the agent is tasked with achieving block stack configurations with varying number of blocks. In this domain, we have four classes of objects: (i) a table, (ii) wooden blocks, (iii) magnetic blocks, and (iv) locations. Magnet blocks will stick to other magnetic blocks, resulting in both blocks moving as the result of a pick action. Locations denote the workspace of available locations that a block can reside in. 

There are two actions, \textsc{pick} and \textsc{place}. Our \textsc{pick} and \textsc{place} actions are parameterized by the target object to act on, the object to pick the object from or place it on, and the corresponding locations of both of those objects. The parameters are sufficient to perform standard block-stacking tasks, but magnetic interactions between blocks necessitate quantification.

A na{\"\i}ve action model is unaware of those interactions. A model learning algorithm must learn from the data to understand that when it picks up a magnetic block that sits upon another magnetic block, both blocks will be lifted. Likewise, when placing down a stack, the bottommost block is going to be placed at the target location instead of the grasped block. Learning agents must be able to model the quantified relationships between the blocks and locations.

The magnet domain features an additional complexity: \emph{disjunctive conditions}. The agent must reason about whether either block is non-magnetic, which involves a disjunctive condition. We extend OHCAM to handle disjunctions with a \emph{greedy set cover} approach: it selects conjunctive hypotheses with zero false positives until every observed effect in the dataset is explained by a true positive from at least one selected hypothesis. This produces a condition in disjunctive normal form. This approach is sufficient to learn a simple disjunction, but could be extended in the future to support more complex case-by-case reasoning.

We show an excerpt of the model learned online in the MagneticBlocks robot hardware environment by OHCAM in Listing~\ref{lst:magnet}.

\begin{listing}[t]
\begin{lstlisting}[language=PDDL]
(:action pick-block
:parameters ( ?target ?support 
    ?targetloc ?supportloc)
:precondition (and
    (on ?target ?support)
    ...
)
:effect (and
    ; regular blocks world
    (grasping ?target) (gripper-full)
    (not (at-loc ?target ?targetloc))
    ; conditional effect when both magnetic
    (when (and (plastic ?target) 
               (plastic ?support))
        (not (at-loc ?support ?supportloc)))
    ; disjunction: one of them is not magnetic
    (when (or (not (plastic ?target)) 
              (not (plastic ?support)))
        (not (on ?target ?support)))
    ; UQV for the block below the support
    (forall (?uqv0 - block)
        (when (and (plastic ?target) 
               (not (wooden ?support)))
            (not (on ?support ?uqv0))))
    ; UQV for the location below the support when magnetic
    (forall (?uqv0 - location)
        (when (and 
                (loc-above ?supportloc 
                    ?uqv0) 
                (plastic ?support) 
                (plastic ?target))
            (not (obstructed-above ?uqv0))))
    ...
\end{lstlisting}
\caption{Abbreviated and annotated excerpt of learned MagnetBlocks \textsc{pick} action with conditional and quantified effects.}
\label{lst:magnet}
\end{listing}

\paragraph{SpongeWorld}
The sponge world domain simulates a household cleaning task in which the primary objective is to clean a dirty plate. In this domain, we have a variety of objects: a table, targets or `coasters' that objects rest on, sponges, cans, messes, and water spills. A blue sponge is always soaked, and wiping a plate with the wet sponge will result in excess water being on the plate. Furthermore, wiping the plate when the can is on the plate will result in the can being knocked on to the table. 

\begin{listing}[t]
\begin{lstlisting}[language=PDDL]
(:action wipe
    :parameters (?grasped ?mess ?target)
    :precondition (and
        (grasping ?grasped)
        (on ?mess ?target)
        ...
    )
    :effect (and
        ; cleans original mess
        (not (dirty ?target))
        (not (on ?mess ?target))
        ; adds spill when sponge is wet
        (forall (?uqv0 - spill)
            (when (wet ?grasped)
                (on ?uqv0 ?target)))
        ; knocks cans off plate
        (forall (?uqv0 - can)
            (when (on ?uqv0 ?target)
              (not (on ?uqv0 ?target))))
        ; knocks cans onto table
        (forall (?uqv0 - table ?uqv1 - can)
            (when
                (and (on ?uqv1 ?target))
                (on ?uqv1 ?uqv0)))))
\end{lstlisting}
\caption{Annotated excerpt of learned SpongeWorld \textsc{wipe} action with conditional and quantified effects.}
\label{lst:sponge}
\end{listing}

In addition to the \textsc{pick} and \textsc{place} actions, we add a third action \textsc{wipe} with three predicates: the mess to wipe, the object with which to wipe the mess, and the surface that the mess sits on.

A na{\"\i}ve action model is unaware that wiping with a wet sponge will leave the water spill on the plate and that wiping the plate with a can on the plate will move the can to the table. Both of these are examples of effects that might be difficult for an expert to foresee or predict ahead of time, or may be accidentally omitted because they are seen as `common knowledge'. An online learning agent can increase task performance by inferring such effects from executed actions. We show the model learned by OHCAM in Listing~\ref{lst:sponge}.

\begin{figure}[t]
    \centering
    \includegraphics[width=\linewidth]{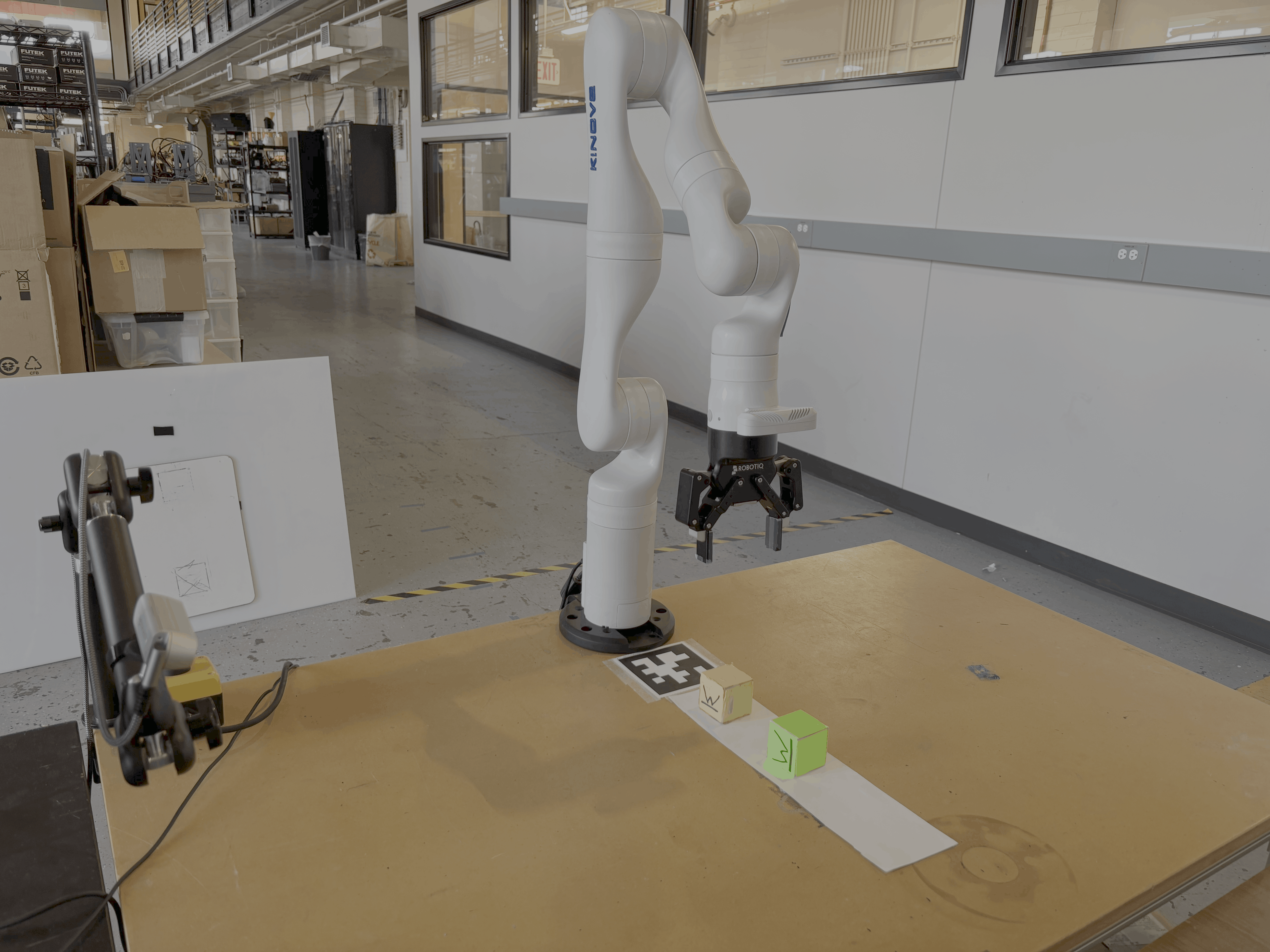}
    \caption{A picture of tabletop manipulation setup picturing the Intel RealSense D435i camera (left) and the Kinova Gen3 Arm with Robotiq 2F85 gripper (center) and AprilTag for camera localization.}
    \label{fig:tabletop_domain}
\end{figure}

\section{Appendix D: Real World Hardware Setup and Evaluation}
This section outlines the components of our robot system that enable model learning and evaluation. Our real world experiments were conducted on a Kinova Gen3 7 DoF manipulator arm with a Robotiq 2F85 two finger gripper, using an Intel RealSense D435i mounted on the side of the workspace for object localization. Figure~\ref{fig:tabletop_domain} shows our setup.

\subsection{Appendix D.1: Hardware Setup}
\paragraph{Perception and Object Localization}
Perception of the objects on the table top is performed using the Intel Realsense D435i camera which we interface with through the realsensepy package, accessing both the RGB and depth streams at 30 frames per second. AM3~\citep{carion2026} segmentation model is used to identify the objects in the scene. To support robust object tracking via language concepts, we assume that there is only one instance of each object per concept in the scene.

After identifying each concept in the RGB image, we use the segmentation masks returned from SAM3 along with the calibrated D435i camera intrinsics to reconstruct the point cloud for each object in the scene. We make one of two assumptions to compute the centroid of the object from the point cloud: (i) the majority of the points lie on the front face which applies to cubes, cans and sponges, or (ii) the majority of the points lie on the top face which applies to plates, cans, water spills, and messes. To obtain the poses of the object in the robot frame, we place a 100mm AprilTag at the base of the robot with an offset along the x-axis of 14cm and use the AprilTags package to compute the camera pose and transform object centroids into the robot frame for manipulation. 

By computing the centroids in this manner, we have made the following assumptions: (i) the orientation of each object is static and (ii) the ideal grasp position can be parameterized solely by the centroid. Given that the aim of our robot system is not to demonstrate complex grasping or perception behavior, we believe that these assumptions are not limiting, and instead serve to demonstrate the performance of our online model learning algorithm in the real world. 

\paragraph{Predicate Grounding from Perception}
The majority of the predicates in our PDDL evaluation domains denote spatial relations among objects. 
We perform predicate grounding programatically by comparing the centroids of the objects. As an example, for the `on' predicate, if the XY center of object A is within the XY bounding box of object B (plus a little slack) and the Z center of object A rests above the Z center of object B plus the size of object A (plus a little slack), then we say that object A is `on' object B.

In this manner, we are able to ground all the predicates in evaluation domains programatically using the centroids computed with SAM3. On an NVIDIA 5070 GPU with 12GB of VRAM, we are able to run the SAM3 segmentation model at 5Hz. Given the execution time of our robot skills, this inference speed was sufficient to ensure that all predicates were accurately grounded during action selection. 

\paragraph{Parameterized Manipulation Skills}
In both of our domains, the model has parameterized actions such as `\textsc{pick}(object A, object B)', interpreted as `Pick object A from object B'. However, these actions from the action plan are agnostic to the position of the object in the real world. With our perception system, we track each object in the real world at 5Hz. Then, when an action is called, we resolve the object names to their centroids and use these centroids to parameterize the \textsc{pick}, \textsc{place} and \textsc{wipe} skills that are present in our two real world domains. 

All of our robot manipulation skills are stateful: we combine a series of motions together to achieve the desired result. These motions are open loop once the skill is parameterized and initiated. To communicate with the Kinova Gen3 robot, we used the Kinova Kortex API Python bindings which enable control at 20Hz over TCP protocol. The Kinova Kortex API exposes a cartesian end effector controller, enabling linear motion between poses of the end effector. The inverse kinematics are resolved on the robot hardware, which in turn controls the necessary joint torques to achieve the desired pose subject to position and orientation velocity limits.

\paragraph{Noise in Perception and Execution}
When an object is misidentified or unidentified by SAM3, the predicate grounding may be erroneous. This can also occur due to partial occlusion of the gripper, although we designed our environment to result in minimal occlusion. The majority of this noise did not impact the successful execution of skills, but was noise that our learning algorithms had to deal with when differentiating between noise and observed conditional effects. 

Execution noise was rarely observed within the system. Infrequently, a sponge would slip out of the gripper if it was too wet, resulting in the pick or place actions to fail prematurely. Most execution failures could be traced to perception failures, either resulting from a noisy depth map and incorrect centroid, or a misidentified object. Overall, we observed around 5\% action failures during data collection resulting from perception noise. This noise represents realistic conditions and challenges in the real world that these learning algorithms must overcome. 

\subsection{Appendix D.2: Robot Data Collection and Evaluation}

\paragraph{Data Collection:} Using our robot system described in \textit{Appendix D.1}, we collected data in the \textit{MagnetBlocks} and \textit{SpongeWorld} domains (see Figure~\ref{fig:real_world_data_collection}). Given that the robot can observe irreversible conditional effects which can result in potentially unsafe actions to the robot hardware (placing a block on another block with a magnet block in between results in excessive torque applied to the joints), all of our hardware experiments are conducted with a human supervisor. 

In this setup, the importance of online active learning becomes very clear: data collection is time consuming and so an agent must take actions that maximally reduce uncertainty about the presence of conditional effects within the environment. We collect 250 skill executions in the \textit{MagnetBlocks} and \textit{SpongeWorld} domains. Data collection takes around 45 minutes to 1.25 hours, primarily due to skill execution and environment resetting rather than action selection. We observed 5\% noise present during data collection, resulting in execution failures or perception failures that were outside the action preconditions. 

\begin{figure}[t]
    \centering

    \begin{subfigure}{0.495\linewidth}
        \centering
        \includegraphics[width=\linewidth]{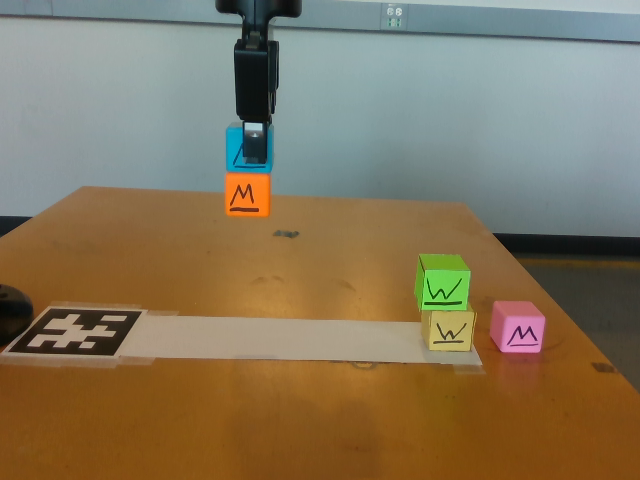}
    \end{subfigure}
    \hfill
    \begin{subfigure}{0.495\linewidth}
        \centering
        \includegraphics[width=\linewidth]{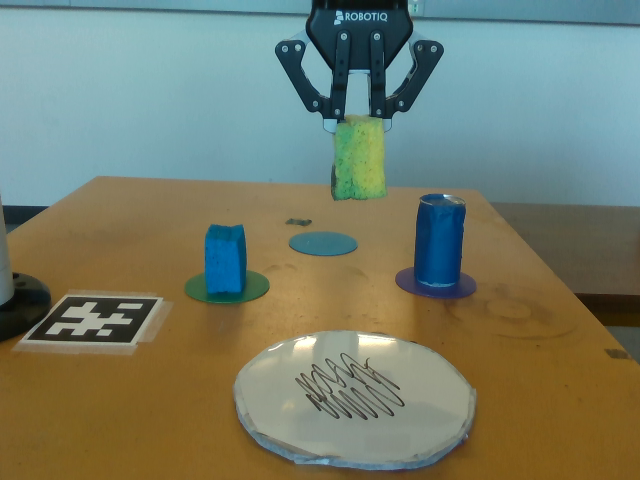}
    \end{subfigure}

    \caption{Data collection in the MagnetBlocks world (left) and the SpongeWorld (right) domains viewed through the Intel RealSense D435i camera which is used for perception. Both cases are showcasing an unmodeled conditional effect. }
    \label{fig:real_world_data_collection}
\end{figure}

\paragraph{Robot Evaluation:}
Using data collected by OHCAM, each of the baselines and OHCAM learned a PDDL model to assess their ability to learn conditional effects from noisy, real-world data. With these learned models, we then evaluated each model on a suite of ten tasks in \textit{MagnetBlocks}, seven of which require knowledge of conditional effects to solve, and five tasks in \textit{SpongeWorld}, four of which require knowledge of conditional effects to solve. We use the FastDownward planner to plan over these learned models and then execute the plan on robot hardware.

We run each task once per learned model. If a task fails due to perception or execution noise, we rerun the plan, otherwise we report the success or failure of the plan execution. We take this approach to evaluate OHCAM's ability to learn in the presence of noisy observations and sample efficiently on real world hardware; not to evaluate the accuracy of our robot system's execution. Nonetheless, during data collection across both domains, we only had to intervene and rerun twice for a skill execution failure, indicating that our robot system is relatively reliable. The learned models and hardware video can be found in the supplemental material folder.

\end{document}